\documentclass[a4,journal]{IEEEtran}
\usepackage{amsmath,amssymb,amsfonts,amsthm} 
\usepackage{algorithmic}
\usepackage{algorithm}
\usepackage{array}
\usepackage{textcomp}
\usepackage{stfloats}
\usepackage{graphicx}
\usepackage{cite}
\usepackage{balance}
\newtheorem{theorem}{Theorem}
\newtheorem{lemma}{Lemma} 
 
\newtheorem{assumption}{Assumption}

\newcommand{\R}{\mathbb{R}}
\newcommand{\N}{\mathbb{N}}
\newcommand{\grph}{\mathcal{G}}     % Graph (interconnection among servers)
\newcommand{\srvr}{\mathcal{S}} 
\newcommand{\clnt}[1]{\mathcal{C}_{#1}}

\newcommand{\nrm}[1]{\Vert #1 \Vert} 
\newcommand{\nrmsq}[1]{\Vert #1 \Vert^2}

\newcommand{\one}{\mathbf{1}} 
\newcommand{\zro}{\mathbf{0}}

\begin{document}

\title{Distributed Federated Learning by Alternating Periods of Training}

\author{Shamik Bhattacharyya, Rachel Kalpana Kalaimani}
        % <-this % stops a space
%\thanks{This paper was produced by the IEEE Publication Technology Group. They are in Piscataway, NJ.}% <-this % stops a space
%\thanks{Manuscript received April 19, 2021; revised August 16, 2021.}}

%\markboth{January~2026}%

\maketitle

\begin{abstract}
Federated learning is a privacy-focused approach towards machine learning where models are trained on client devices with locally available data and aggregated at a central server. 
However, the dependence on a single central server is challenging in the case of a large number of clients and even poses the risk of a single point of failure. 
To address these critical limitations of scalability and fault-tolerance, we present a distributed approach to federated learning comprising multiple servers with inter-server communication capabilities. While providing a fully decentralized approach, the designed framework retains the core federated learning structure where each server is associated with a disjoint set of clients with server-client communication capabilities.
We propose a novel DFL (Distributed Federated Learning) algorithm which uses alternating periods of local training on the client data followed by global training among servers. We show that the DFL algorithm, under a suitable choice of parameters, ensures that all the servers converge to a common model value within a small tolerance of the ideal model, thus exhibiting effective integration of local and global training models. Finally, we illustrate our theoretical claims through numerical simulations. 
\end{abstract}


\section{Introduction} \label{sec:intro}
\IEEEPARstart{T}{he} introduction of federated learning in \cite{FedL_Google_2016} opened a new avenue of machine learning where a centralized source of training data was no longer necessary. In federated learning, each client trains a local model using its own data, and all these client models are periodically aggregated by a central server into the global model. This decentralized approach focused on privacy-preserving machine learning is finding applications in improving the user experience of smartphones \cite{iot_Survey}, advancing digital health applications \cite{med_loAdaboost}, banking applications \cite{fedl_banking} and promising more in the near future.

The rise in awareness of user privacy, which has led to users being more vigilant about sharing data, has generated increasing interest in federated learning. As the sensitive user data remains with the client and is no longer needed to be shared with the server, privacy is inherently ensured. Moreover users gain more control over their data that is used to train the model on their devices \cite{mobKeybrd_hard2019}. 
From its inception, federated learning has been mostly focused towards applications over mobile edge devices as clients. Various algorithms have been proposed to address and improve upon different aspects like system heterogeneity \cite{FedLGA_Cyber_24}, communication efficiency \cite{fedl_comm_smc23}, adversarial attacks \cite{fedl_adv_cyber24} etc. Recently federated learning has been considered for the advancement of medical and healthcare applications \cite{fedL_med_nlm}. 

We focus on such applications where clients are hospitals and other medical organizations with sensitive data of medical records. Moreover, these organizations can be expected to have sufficient computation and communication resources to train on large data-sets and also communicate periodically with the server. For designing more effective models, it would be desirable to have clients spread across different regions and countries. In such scenarios having a single central server to address all clients may not be feasible due to official protocols, geo-political issues, etc. This is where allowing each region or country to have their own server addressing the local clients while having the ability to communicate among the servers of neighbouring regions can provide a collaborative solution towards the advancement of global healthcare. 
This motivates us to develop a federated learning approach considering multiple servers. Moreover, the dependence of existing federated learning approaches on a single central server may be undesirable in the case of a large number of clients and could also pose the risk of being a single point of failure \cite{FedL_openProbs_2021}.

We present here a distributed approach to federated learning using multiple servers where the servers are able to communicate among themselves. Each server has a corresponding set of clients that periodically communicate with the server. We term this as \textbf{distributed federated learning}. Towards fully decentralized federated learning, peer-to-peer learning has been considered where the communication with server is replaced with communication among the clients \cite{pr2pr_brnTrrnt_agRoy2019}. Although this approach eliminates the need for a central server, it may still need some central authority as mentioned in \cite{FedL_openProbs_2021}. 
The approach of using multiple servers has recently been studied as hierarchical federated learning \cite{hfl_Liu_icc2020},\cite{hfl_medIoT_Zhou_23}. While it does allow for local servers to address disjoint groups of clients, inter-server communication is not considered. Moreover, it still needs a central server to periodically perform the final aggregation of the models. 
To ensure that a global parameter model is estimated that best suits the data across all clients and all the servers agree over it, we design a novel \textbf{DFL (Distributed Federated Learning)} algorithm. The proposed DFL algorithm uses repeated cycles of training on the client data using a gradient based approach for a certain period followed by a consensus approach via inter-server communication among servers to arrive at a common global model for the remaining time. 
In consensus algorithms \cite{con_TnF_survey} and consensus-based distributed optimization algorithms \cite{disOpti_Nedic_tac09}, the agents continuously follow an iterative update law to arrive at the common estimate. In case of the DFL algorithm the consensus update law followed by the servers is interspersed with periods of clients following the gradient based update law and their model aggregates being incorporated by the respective servers. 
The challenge of the servers arriving at a global model estimate while periodically incorporating their corresponding client model aggregates is effectively managed by the DFL algorithm as established through its theoretical convergence. In particular, we show that the DFL algorithm, under suitable choice of parameters, ensures that all the servers converge to a model value within a small tolerance from the ideal model. 

Our main contributions in this paper are listed below. 
\begin{itemize}
\item We design a distributed approach to federated learning comprising multiple servers with the capability of communication among neighbouring servers. This addresses the critical limitations of scalability and fault-tolerance of single-server or hierarchical federated learning models. A corresponding disjoint set of clients is associated to each of these servers, where the data for training the model parameters are available locally with the clients.

\item We propose a novel algorithm, the \textbf{DFL} algorithm, designed to ensure that all servers eventually agree on a common model parameter value that will perform well across all client devices (Algorithm \ref{alg:dfl}). The novel aspect of the algorithm lies in the periodic shifting between local training across clients and the global training among servers. The intervals of clients training their models on the locally available data is interspersed with periods where the servers communicate among themselves to achieve consensus over a common acceptable global model. 

\item We establish convergence guarantees for the DFL algorithm. 
The periodic nature of the proposed algorithm along with the shifting between local and global training requires a different approach in establishing the convergence proof in comparison to the conventional consensus based distributed approaches. 
While the algorithm is based on  gradient descent, we derive the step size that ensures convergence for the DFL algorithm. We observe that this is dependent on the number of iterations that is performed on each client before the server iterations. We show that the DFL algorithm, with an appropriate choice of parameters, ensures that the prediction model across all servers is within a certain tolerance $\epsilon$ from the value of the ideal model (Theorem \ref{thm:mainThm}). 

\end{itemize} 

\noindent \textit{Notations :} $\R$ denotes the set of \emph{real} numbers, and $\R^N$ represents the $N$-dimensional Euclidean space. For any set $\mathcal{S}$, the \textit{cardinality} of the set is denoted by $|\mathcal{S}|$. $\one := (1,1,\hdots,1)$ and $\zro := (0,0,\hdots,0)$, of appropriate dimensions. For a real-valued vector $v$, $v'$ denotes the \emph{transpose} of the vector and $||v||$ denotes its $l_2$-norm. Similarly, for a real-valued matrix $V$, $V'$ denotes the \emph{transpose} of the matrix, and $||V||$ denotes its \emph{spectral norm}.  

The organization of the paper is as follows. Section-\ref{sec:ProbFrmn} discusses the details of the problem which we refer to as distributed federated learning. Section-\ref{sec:Res} starts with a discussion on the details of the proposed DFL algorithm, followed by some intermediate results, which are then used to finally present our main result. A numerical simulation is presented in Section-\ref{sec:NumRes} to validate the performance of the DFL algorithm. Finally, the conclusions are presented in Section-\ref{sec:Con}.

%% ~~~~~~~~~~~~~~~  Problem Formulation  ~~~~~~~~~~~~~~~ %%
\section{Problem Formulation} \label{sec:ProbFrmn}
%% ~~~~~~~~~~~~~~~  PF:System Model  ~~~~~~~~~~~~~~~ %%
\subsection{System Model} \label{sec:PrbFrmn_SysMdl}
We consider a \textit{distributed} federated learning architecture consisting of $M$ servers, represented by the set $\mathcal{S}$. Each server $i$ has a corresponding set of $N$ clients, $\mathcal{C}_i$ which periodically communicate with the server. Moreover, each server can communicate with its neighbouring servers, and this communication is represented by an undirected graph $\mathcal{G}:(\mathcal{V},\mathcal{E})$. Here $\mathcal{V}$ denotes the set of vertices of the graph, representing the servers, and $\mathcal{E}$ denotes the edges of the graph, representing the bidirectional communication links among pairs of neighbouring servers. So, $\mathcal{V} = \mathcal{S}$, and $|\mathcal{V}| = M$. We use the following standard assumption on the graph $\mathcal{G}$ which helps in ensuring all the servers achieve consensus.  
\begin{assumption} \label{asmpn:grphCnnctd}
    The graph $\grph$ is connected.
\end{assumption} 
\begin{figure}
    \centering
    \includegraphics[scale=0.2]{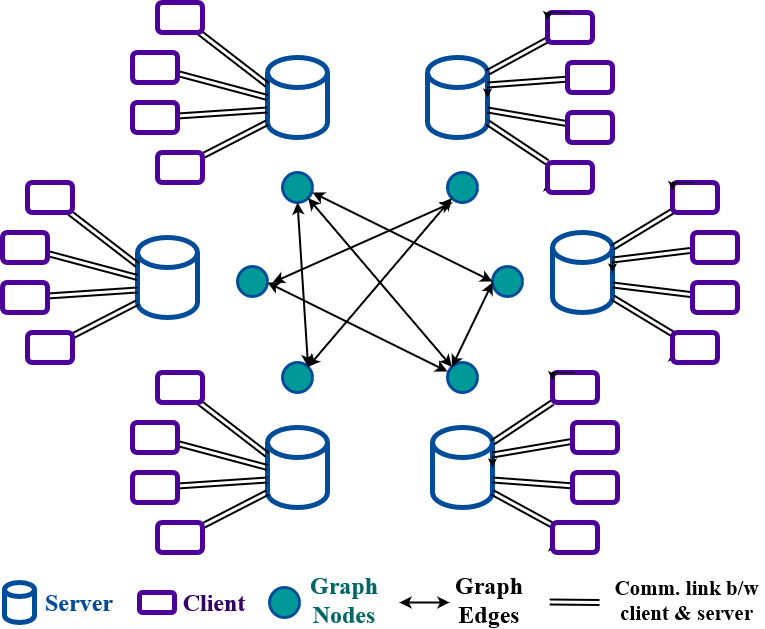}
    \caption{System model example with M=6 and N=4}
    \label{fig:sysModel}
\end{figure} 

We present a sample system model in Fig.\ref{fig:sysModel} comprised of 6 servers and 4 clients per server. The double-line link between the server and its corresponding clients represents the periodic communication between them to share their updated model parameters. The graph at the centre represents the communication among the servers - nodes of the graph symbolize the servers, while the edges indicate the communication links between the servers. 
%% ~~~~~~~~~~~~~~~  PF:Distributed FedL  ~~~~~~~~~~~~~~~ %%
\subsection{Distributed Federated Learning} \label{sec:PrbFrmn_dFedL}
Federated learning is an approach to training a machine learning model at a central server using the data that is locally available across multiple clients. The main idea is to ensure privacy by not requiring to move the corresponding data out of the client devices. Here we introduce the idea of \textit{distributed} federated learning, where instead of a single central server, we have multiple servers, each associated with a set of clients.  

For any server $i$, each client $j$ associated to the server has its corresponding set of $D$ data points, $\mathcal{D}^{ij} = \{ (x_1,y_1),(x_2,y_2), \hdots, (x_D,y_D) \}$, where $x_k \in \R^d$ and $y_k \in \R$ for all $k = 1, \hdots, D$.  The empirical risk in prediction using the model parameter $w \in \R^d$, over the locally available data set $\mathcal{D}^{ij}$ of client $j$, is given by $f^{ij}(w) := 1/D \sum_{k=1}^D l(w;(x_k,y_k)) $. Here $l(.)$ is the predefined loss function across all clients. So the net empirical risk associated with any server $i$ is given by $f^i (w) := 1/N \sum_{j=1}^N f^{ij}(w)$. The goal then is to find a suitable prediction model parameter $w$ that will perform well on all client devices across all servers. This goal is essentially the solution of the following distributed optimization model : 
\begin{equation} \label{eq:prblmStmnt}
    \min_w f(w) \triangleq \frac{1}{M} \sum_{i=1}^{M} f^i(w) \triangleq \frac{1}{M} \sum_{i=1}^{M} \frac{1}{N} \sum_{i=1}^{N} f^{ij}(w)
\end{equation}

We introduce the following assumptions on the empirical risk functions associated with the clients. These assumptions on the objective function are commonly used in the literature of federated learning \cite{AMitra_NeurIPS21} and distributed optimization \cite{DistrOpti_Nedic_SIAM}. 
\begin{assumption} \label{asmpn:funcn_fij}
    The risk functions $f^{ij}(.)$ are $\mu$-strongly convex and $L$-smooth, for all $j \in \clnt{i}$, $\forall i \in \srvr$.
\end{assumption} 
\begin{assumption} \label{asmpn:gradBound}
    The gradient of the risk functions across all clients is bounded, i.e 
    \begin{equation} \label{eq:gradBound}
        \nrm{\nabla f^{ij}(w)} \leq \theta \text{ , for all } j \in \clnt{i}, \forall i \in \srvr
    \end{equation}
\end{assumption}

%% ~~~~~~~~~~~~~~~  Results  ~~~~~~~~~~~~~~~ %%
\section{Results} \label{sec:Res}
%% ~~~~~~~~~~~~~~~  R:Algorithm  ~~~~~~~~~~~~~~~ %%
\subsection{Algorithm} \label{sec:Res_Algo} 
We present our proposed \textit{DFL} algorithm designed to find a prediction model parameter $w$ by solving the optimisation problem in \eqref{eq:prblmStmnt}.
\begin{figure*}
    \centering
    \includegraphics[width=0.8\textwidth,height=5cm]{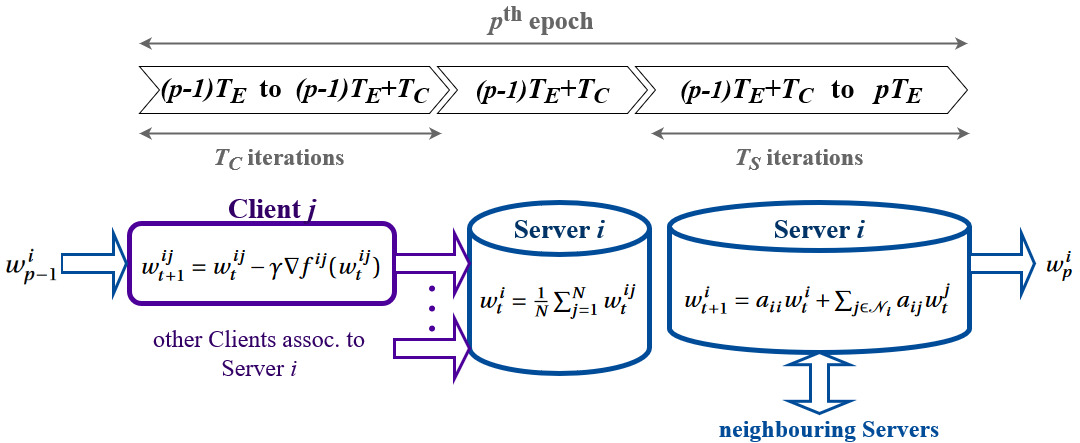}
    \caption{Timeline representation of client and server iterations over 1 epoch of the DFL algorithm}
    \label{fig:alg_1epoch}
\end{figure*}

First, we specify one epoch to be of $T_E \in \N$ time steps or iterations, which consists of $T_C$ iterations of client computation followed by $T_S$ iterations of server computations. Thus, $T_C + T_S = T_E$. The DFL algorithm consists of two main parts of computations: the client side followed by the server side, repeated over every epoch. 

Now consider any $p$-th epoch, $p \in \N$. Firstly, for each server $i \in \srvr$, every client $j \in \clnt{i}$ maintains its own local model parameter $w^{ij}_t$, and updates it using the following : 
\begin{equation} \label{eq:UL_client}
    w^{ij}_{t+1} = w^{ij}_t - \gamma \nabla f^{ij}(w^{ij}_t) %\text{ , }  
\end{equation}
where $(p-1)T_E \leq t < (p-1)T_E + T_C$.
Here we consider that the clients use a common constant step-size parameter, $\gamma$. This client side computation is performed in parallel across all clients $\bigcup_{i=1}^M \clnt{i}$ for $T_C$ iterations. After that, every client communicates its latest updated model parameter value to its corresponding server. Each server $i \in \srvr$ then updates its own model parameter $w^i_t$ by taking an average of all the values received from its clients as :
\begin{equation} \label{eq:UL_clntSrvr}
    w^{i}_{t} = \frac{1}{N}\sum_{j=1}^N w^{ij}_{t} %\text{ , } 
\end{equation}
where $t = (p-1)T_E + T_C$.
With the value from \eqref{eq:UL_clntSrvr} as the initial value, each server $i$ updates its model parameter by using the following update law :
\begin{equation} \label{eq:UL_server}
    w^{i}_{t+1} = a_{ii} w^i_t + \sum_{j \in \mathcal{N}_i} a_{ij} w^{j}_t 
\end{equation}
where $(p-1)T_E + T_C \leq t < pT_E$, and $\mathcal{N}_i$ represents the neighbors of server i. The scalars $a_{ij} \in \R$ are weights assigned by the server $i$ to its own and neighbours' values, such that it obeys the following properties with $0< \alpha < 1$ : 
\begin{equation} \label{eq:aij_wtsProp} 
    a_{ij} \begin{cases} 
    > \alpha &\text{ if } j \in \mathcal{N}_i \cup \{i\}, \\ 
    = 0 &\text{ otherwise} 
    \end{cases} ; 
    \sum_{j=1}^M a_{ij} = 1 ; \sum_{i=1}^M a_{ij} = 1
\end{equation}
All the $M$ servers perform the computation of \eqref{eq:UL_server} in parallel for $T_S$ iterations. With this, after a total of $T_E$ iterations consisting of both client and server side computations, the $p$-th epoch concludes. Finally, each server $i$ communicates its latest model parameter update at the end of $p$-th epoch, $w^i_p$ to all its clients $\clnt{i}$. All these details of the DFL algorithm is summarised in a pseudo-code format in Algorithm \ref{alg:dfl}.  
Alongside it a pictorial timeline representation of the iterations in any $p$th epoch is shown in Fig.\ref{fig:alg_1epoch}. 

 \begin{algorithm} 
\caption{DFL : Distributed Federated Learning} 
\label{alg:dfl}
 \textbf{Given} : $M$ servers, $N$ clients/server, graph $\mathcal{G}$, $T_C$ client iterations, $T_S$ server iterations, parameters $\mu, L, \gamma$, $\theta$ 
\\ \textbf{Initialize} : $w_0 \in \R^d$, shared across all servers and clients
\\ \textbf{for} $p = 1,2,\hdots $ \textbf{do} 
\\ \hspace*{0.3cm} \textbf{parallel for all servers $i \in \srvr$ do} 
\\ \hspace*{0.6cm} \textbf{parallel for all clients $j \in \clnt{i}$ do} 
\\ \hspace*{0.9cm} \textbf{for} $t = (p-1)T_E : (p-1)T_E+T_C$ \textbf{do} 
\\         \hspace*{1.2cm} Client computes : $w^{ij}_{t+1} = w^{ij}_t - \gamma \nabla f^{ij}(w^{ij}_t)$
\\ \hspace*{0.9cm} \textbf{end for}
\\ \hspace*{0.9cm} Client communicates : sends $w^{ij}_t$ to server $i$
\\ \hspace*{0.6cm} \textbf{end parallel for} 
    \\ \hspace*{0.6cm}    Server computes : $w^{i}_{t} = \frac{1}{N}\sum_{j=1}^N w^{ij}_{t}$
\\ \hspace*{0.6cm} \textbf{for} $t = (p-1)T_E+T_C : pT_E$ \textbf{do} 
\\ \hspace*{0.9cm} Server communicates : sends $w^i_t$ to neighbors $\mathcal{N}_i$
        \\ \hspace*{0.9cm} Server computes :  $w^{i}_{t+1} = a_{ii} w^{i}_t + \sum_{j \in \mathcal{N}_i} a_{ij} w^{j}_t$
\\ \hspace*{0.6cm} \textbf{end for}
\\ \hspace*{0.6cm} Server communicates : sends $w^i_t$ to clients $\mathcal{C}_i$
\\ \hspace*{0.3cm} \textbf{end parallel for}
\\ \textbf{end for}
\\ \textbf{Output : } $w^i_p$ for all $i \in \mathcal{S}$ 
\end{algorithm}

%% ~~~~~~~~~~~~~~~  R:Intermediate Results  ~~~~~~~~~~~~~~~ %%
\subsection{Intermediate Results} \label{sec:Res_InterRes}
Here we present some intermediate results related to the DFL algorithm that help us to finally establish our main result in Section \ref{sec:Res_MnRes}. 
First we define two matrices as : $W_t \in \R^{M \times d}$, $W_t := [(w^1_{t})' ; (w^2_{t})'; \hdots ; (w^M_{t})'], t \in \N$, and $A \in \R^{M \times M}$, where the $(i,j)$th element of $A$ is $a_{ij}$ from \eqref{eq:aij_wtsProp}. Now we rewrite the update law in \eqref{eq:UL_server} considering all the servers as :
\begin{equation} \label{eq:UL_server_matForm}
    W_{t+1} = A W_t .
\end{equation} 
The following lemma establishes that after any given epoch $p$, the distance between the model parameter estimate of any server $i$, $w^i_p$, and the common average model parameter value across all servers, $\Bar{w}_p$, is always bounded. It also shows that this bound decreases with increasing number of epochs. 

%% ~~~~~~~~~~~~~~~  R_IR:Lemma 1  ~~~~~~~~~~~~~~~ %%
\begin{lemma} \label{lem:srvrAvgBound}
   Suppose Assumptions \ref{asmpn:grphCnnctd} and \ref{asmpn:gradBound} hold. Then the DFL algorithm ensures that the difference between the model parameter estimate of any server $i$, $w^i_{p}$ and the global average model parameter estimate across all servers, $\Bar{w}_{p}$ is bounded for every epoch $p$.
   Specifically, for all $i \in \srvr$, $p \in \N$
    \begin{equation} \label{eq:srvr4mAvgBound}
        \nrm{w^i_{p} - \Bar{w}_{p}} \leq \sigma_A^p \nrm{W_{0} - \one \Bar{w}'_0} + \sqrt{M} T_C \theta \gamma \sigma_A / (1 - \sigma_A)
    \end{equation} 
    where $\sigma_A = \nrm{A^{T_S} - \frac{1}{M} \one \one'}$ and $\gamma$ is as in \eqref{eq:UL_client}. 
\end{lemma}
\begin{proof}
Consider any $p+1$-th epoch, $p \in \N$. For the first $T_C$ iterations, from $pT_E + 1$ to $pT_E + T_C$, the clients across all servers perform the gradient descent step in parallel. So for any client $j$ associated to some server $i$ we have : 
\begin{align}
     \label{eq:Lem1_1} w^{ij}_{pT_E+T_C} &= w^{ij}_{pT_E} - \gamma \sum_{\tau = p T_E}^{pT_E+T_C-1} \nabla f^{ij}(w^{ij}_{\tau})
\end{align}
At the the end of $p$-th epoch, server $i$ communicates its latest model parameter value $w^i_p$ to all its corresponding clients. So at the starting of the $p+1$-th epoch, the initial model parameter of client $j$ is $w^{ij}_{pT_E} = w^{ij}_{p} = w^{i}_{p}$. Using this in \eqref{eq:Lem1_1} we get
\begin{equation} \label{eq:clnt_aftrTc}
    w^{ij}_{pT_E+T_C} = w^{i}_{p} - \gamma \sum_{\tau = p T_E}^{pT_E+T_C-1} \nabla f^{ij}(w^{ij}_{\tau}).
\end{equation}
After $T_C$ iteration, the model parameter estimate of the client parameters is communicated to the corresponding servers, where each server updates its own model parameter estimate by taking an average of its clients' estimates.
\begin{align}   
   \nonumber w^{i}_{pT_E + T_C} &= \frac{1}{N}\sum_{j=1}^N w^{ij}_{pT_E + T_C} \\ 
   \nonumber  &= w^i_p - \frac{\gamma}{N}\sum_{j=1}^N \sum_{\tau = p T_E}^{pT_E+T_C-1} \nabla f^{ij}(w^{ij}_{\tau}) \\ 
    \therefore w^{i}_{pT_E + T_C}  &= w^i_p - \gamma g^i_p      \label{eq:srvr_aftrTc}
\end{align} 
where $g^i_p := (1/N)\sum_{j=1}^N \sum_{\tau = p T_E}^{pT_E+T_C-1} \nabla f^{ij}(w^{ij}_{\tau})$. 

\vspace{0.2cm}
Let $G_t := [(g^1_{t})' ; (g^2_{t})'; \hdots ; (g^M_{t})']$. Then from \eqref{eq:srvr_aftrTc}, considering all the servers we can say 
\begin{equation}
    W_{pT_E+T_C} = W_{p} - \gamma G_{p} ,
\end{equation}
For $T_S$ iterations, from $pT_E + T_C + 1$ to $pT_E + T_C + T_S$, the servers perform the consensus update in \eqref{eq:UL_server_matForm} which can be represented as : 
\begin{align}
    \nonumber W_{pT_E + T_C + T_S} &= A W_{pT_E + T_C + T_S - 1} = \hdots = A^{T_S} W_{pT_E + T_C}\\ 
    \label{eq:srvrs_Con_p_p1} \therefore W_{p+1} &= A^{T_S} (W_{p} - \gamma G_{p})
\end{align}
Then the difference of the servers' model parameter estimates from the common average across all servers, using \eqref{eq:srvrs_Con_p_p1}, can be written as : 
\begin{align}
    \nonumber W_{p+1} &- \one \Bar{w}'_{p+1} = W_{p+1} - \one (\frac{1}{M} \one' W_{p+1}) \\ 
    \nonumber &= (I - \frac{1}{M} \one \one') A^{T_S} (W_{p} - \gamma G_{p})  \\ 
    \label{eq:diff_srvrMat4mAvg} &= (A^{T_S} - \frac{1}{M} \one \one') (W_{p} - \one \Bar{w}'_p) - (A^{T_S} - \frac{1}{M} \one \one') \gamma G_{p} 
\end{align} 
where for the last step we use $\one' A^{T_S} = \one'$. Now applying the spectral norm to both sides of \eqref{eq:diff_srvrMat4mAvg} and using its sub-multiplicative property we get
\begin{equation*}  
    \nrm{W_{p+1} - \one \Bar{w}'_{p+1}} \leq \sigma_A \nrm{W_{p} - \one \Bar{w}'_p} + \sigma_A \gamma \nrm{G_{p}}.  
\end{equation*}
 where $\sigma_A := \nrm{A^{T_S} - \frac{1}{M} \one \one'}$.
\begin{equation} \label{eq:diffNrm_srvrMat4mAvg}
    \therefore \nrm{W_{p} - \one \Bar{w}'_{p}} \leq \sigma_A^p \nrm{W_{0} - \one \Bar{w}'_0} +  \gamma \sum_{l=0}^{p-1} \sigma_A^{p-l} \nrm{G_{l}} ,
\end{equation} 
Now we proceed to first derive a bound for each row of $G_p$ for any epoch $p$, and then use it to get a bound for $\nrm{G_p}$. 
\begin{align}
    \nonumber \nrm{g^i_p}^2 &= \nrm{(1/N)\sum_{j=1}^N \sum_{\tau = p T_E}^{pT_E+T_C-1} \nabla f^{ij}(w^{ij}_{\tau})}^2 \\ 
     \nonumber &\overset{(a)}{\leq} (1/N)\sum_{j=1}^N \nrm{ \sum_{\tau = p T_E}^{pT_E+T_C-1} \nabla f^{ij}(w^{ij}_{\tau})}^2 \\ 
     \nonumber &\overset{(b)}{\leq} (1/N)\sum_{j=1}^N T_C \sum_{\tau = p T_E}^{pT_E+T_C-1} \nrm{ \nabla f^{ij}(w^{ij}_{\tau})}^2 \\ 
     \nonumber &\overset{(c)}{\leq} (1/N)\sum_{j=1}^N T_C \sum_{\tau = p T_E}^{pT_E+T_C-1} \theta^2 \\
    \label{eq:gip_sqnrm_bnd} \therefore \nrm{g^i_p}^2 &\leq T_C^2 \theta^2
\end{align} 
where $(a), (b)$ follow from the convexity of the square norm, and $(c)$ uses \eqref{eq:gradBound} from Assumption \ref{asmpn:gradBound}. Using \eqref{eq:gip_sqnrm_bnd} we get
\begin{align} \label{eq:GpMatrx_Norm}
    \nrm{G_p} \leq ||G_p||_F = \sqrt{\sum_{i=1}^M \nrm{g^i_p}^2} \leq \sqrt{\sum_{i=1}^M T_C^2 \theta^2} = \sqrt{M} T_C \theta
\end{align}
Using \eqref{eq:GpMatrx_Norm} in \eqref{eq:diffNrm_srvrMat4mAvg} we get
\begin{align}
    \label{eq:diffNrm_allSrvrs4mAvg} \nrm{W_{p} - \one \Bar{w}'_{p}} &\leq \sigma_A^p \nrm{W_{0} - \one \Bar{w}'_0} + \sqrt{M} T_C \theta \gamma \sum_{l=0}^{p-1} \sigma_A^{p-l} 
\end{align} 
As $A$ is a doubly stochastic matrix with non-negative entries, we have $\sigma_A < 1$. Using this in \eqref{eq:diffNrm_allSrvrs4mAvg} provides the required result \eqref{eq:srvr4mAvgBound}.
\end{proof}
The next result is inspired from \cite[Lemma 6]{AMitra_NeurIPS21}, which is then used to establish the lemmas that follow. 
%% ~~~~~~~~~~~~~~~  R_IR:Lemma 2  ~~~~~~~~~~~~~~~ %%
\begin{lemma} \label{lem:wvDiffSq}
    Suppose $f(.)$ satisfies Assumption \ref{asmpn:funcn_fij}. Then, for any $0 \leq \eta\leq 1/L$, and any two points $v, w \in \R^d$, we have 
    \begin{equation} \label{eq:wvDiff}
        \nrm{w - v - \eta(\nabla f(w) - \nabla f(v))} \leq \lambda \nrm{w - v}
    \end{equation}
   where $\lambda = \sqrt{1 - \eta \mu}$.
\end{lemma}
\begin{proof}
    For any $v, w \in \R^d$ : 
    \begin{align}
        \nonumber \nrmsq{w &- v - \eta(\nabla f(w) - \nabla f(v))} \\ 
        \nonumber &= \nrmsq{w - v} + \eta^2 \nrmsq{\nabla f(w) - \nabla f(v)} \\ 
        \nonumber &\hspace{0.5cm} - 2\eta \langle w - v, \nabla f(w) - \nabla f(v) \rangle \\ 
        \nonumber &\overset{(a)}{\leq} \nrmsq{w - v} + \eta^2 L \langle w - v, \nabla f(w) - \nabla f(v) \rangle \\ 
        \nonumber &\hspace{0.5cm} - 2\eta \langle w - v, \nabla f(w) - \nabla f(v) \rangle \\ 
        \nonumber &= \nrmsq{w - v} - \eta(2 - \eta L) \langle w-v, \nabla f(w) - \nabla f(v) \rangle \\ 
        \label{eq:lem_wvDiffSq} &\overset{(b)}{\leq} (1 - \eta \mu (2 - \eta L)) \nrmsq{w - v}
    \end{align} 
       where $(a)$ follows from $L$-smoothness and convexity of $f(.)$, and $(b)$ follows from $\mu$-strong convexity of f(.).
    As $\eta \leq 1/L$, we have $1 - \eta \mu (2 - \eta L) \leq 1 - \eta \mu $. Using this in \eqref{eq:lem_wvDiffSq} with $\lambda := \sqrt{1 - \eta \mu}$, we get \eqref{eq:wvDiff}.
\end{proof}

Next we present the following lemma which establishes a bound on how far any client's model parameter value can deviate from its corresponding server's model, within an epoch. This bound is then used to establish the result of the next lemma.
%% ~~~~~~~~~~~~~~~  R_IR:Lemma 3  ~~~~~~~~~~~~~~~ %%
\begin{lemma} \label{lem:clnt4mSrvr}
    The difference of any of the clients' model parameter value from its corresponding server's model parameter, within an epoch, is bounded. Specifically, for any $p \in \N$ and $s \in \{ pT_E + 1, pT_E + 2, \hdots, pT_E + T_C \}$,
    \begin{equation} \label{eq:clnt4mSrvrBound}
        \nrm{w^{ij}_{s} - w^i_p } \leq \gamma T_C \theta.
    \end{equation}
\end{lemma} 
\begin{proof}
For any given $p \in \N$, consider any $s \in \{ pT_E + 1, pT_E + 2, \hdots, pT_E + T_C \}$ :
\begin{align}
    \nonumber \nrm{w^{ij}_{s + 1} &- w^i_p } = \nrm{w^{ij}_{s} - w^i_p - \gamma \nabla f^{ij}(w^{ij}_{s}) } \\ 
        \nonumber &\leq \nrm{w^{ij}_{s} - w^i_p - \gamma (\nabla f^{ij}(w^{ij}_{s}) - \nabla f^{ij}(w^{i}_{p})) } \\ 
        \nonumber &  + \gamma \nrm{\nabla f^{ij}(w^{i}_{p})) } \\ 
        \nonumber &\overset{(a)}{\leq} \lambda \nrm{w^{ij}_{s} - w^i_p} + \gamma \nrm{\nabla f^{ij}(w^{i}_{p}))} \\ 
        \therefore \nrm{w^{ij}_{s} &- w^i_p } \leq \lambda^s \nrm{w^{ij}_{p} - w^i_p} + \gamma \sum_{l=0}^{s-1} \lambda^l \nrm{\nabla f^{ij}(w^{i}_{p}))}
\end{align} 
where $(a)$ follows using \eqref{eq:wvDiff}.
Applying Assumption \ref{asmpn:gradBound}, and using the facts that $\lambda < 1$ and  $w^{ij}_{p} = w^{i}_{p}$, we get \eqref{eq:clnt4mSrvrBound}.
\end{proof} 

Finally we present the next result which shows how the average model parameter value across servers, $\Bar{w}_p$ evolves with every epoch to move closer to the optimal model parameter value $w^*$.
%% ~~~~~~~~~~~~~~~  R_IR:Lemma 4  ~~~~~~~~~~~~~~~ %%
\begin{lemma} \label{lem:avg4mOptml}
    Suppose Assumptions \ref{asmpn:funcn_fij} and \ref{asmpn:gradBound} hold. Then with $\gamma < 1/(\mu T_C)$, the difference of the average estimate across all servers from the optimal value remains bounded. Specifically,  
    \begin{equation} \label{eq:avg4mOptmlBound}
        \nrm{\Bar{w}_{p} - w^*} \leq \Lambda^p \nrm{\Bar{w}_0 - w^*} + Y_0 / (1 - \Lambda)
    \end{equation}
    where $Y_0 = (\gamma T_C )^2 \theta L + (\gamma T_C )^2 \theta L \sqrt{M} \sigma_A/(1 - \sigma_A) + \gamma T_C L \nrm{W_{0} - \one \Bar{w}'_0} $, and $\Lambda = \sqrt{1 - \gamma \mu T_C}$. 
\end{lemma}

\begin{proof}
Consider any $p \in \N$. Then using \eqref{eq:srvr_aftrTc} and the fact that $\nabla f(w^*) = 0$, we can write
\begin{align}
    \nonumber \nrm{\Bar{w}_{p+1} &- w^*} = \nrm{\Bar{w}_p - \gamma \frac{1}{M} \sum_{i=1}^M g^i_\tau - w^* } \\ 
        \nonumber &\leq \nrm{\Bar{w}_p - w^* - \gamma T_C (\nabla f(\Bar{w}_p) - \nabla f(w^*)) } \\ 
        \nonumber &+ \gamma \sum_{MNp} \nrm{\nabla f^{ij}(\Bar{w}_p) - \nabla f^{ij}(w^{ij}_{\tau})} 
\end{align} 
where $\sum_{MNp} := \frac{1}{M} \sum_{i=1}^M \frac{1}{N} \sum_{j=1}^N \sum_{\tau = p T_E}^{pT_E+T_C-1}$. 

\noindent Now using $L$-smoothness of $f^{ij}(.)$ from Assumption \ref{asmpn:funcn_fij} and the result \eqref{eq:wvDiff} from Lemma \ref{lem:wvDiffSq} with $\Lambda := \sqrt{1 - \gamma \mu T_C}$, we get 
\begin{align}
    \nonumber \nrm{\Bar{w}_{p+1} &- w^*} \leq \Lambda \nrm{\Bar{w}_p - w^*} + \gamma \sum_{MNp} L \nrm{w^{ij}_{\tau} - \Bar{w}_{p}} \\ 
    \nonumber &\leq \Lambda \nrm{\Bar{w}_p - w^*} + \gamma L \sum_{MNp} (\nrm{w^{ij}_{\tau} - w^i_p} + \nrm{w^i_p - \Bar{w}_{p}}) \\ 
    \nonumber &\overset{(a)}{\leq} \Lambda \nrm{\Bar{w}_p - w^*} + \gamma L \sum_{MNp} \gamma T_C \theta \\ 
    \nonumber & \hspace{0.3cm} + \gamma L \sum_{MNp} ( \sigma_A^p \nrm{W_{0} - \one \Bar{w}'_0} + \sqrt{M} T_C \theta \gamma \sum_{l=0}^{p-1} \sigma_A^{p-l}) \\
    \nonumber &\overset{(b)}{\leq} \Lambda \nrm{\Bar{w}_p - w^*} + (\gamma T_C)^2 \theta L + \gamma T_C L \sigma_A^p \nrm{W_{0} - \one \Bar{w}'_0} \\ 
    \label{eq:lem4_1} &\hspace{0.3cm} + (\gamma T_C)^2  \theta L \sqrt{M} \sigma_A / (1 - \sigma_A)  
\end{align}
where $(a)$ follows from \eqref{eq:clnt4mSrvrBound} in Lemma \ref{lem:clnt4mSrvr} and \eqref{eq:srvr4mAvgBound} in Lemma \ref{lem:srvrAvgBound}, and $(b)$ follows from the fact that $\sigma_A < 1$. 

\noindent Let $Y_t := (\gamma T_C)^2 \theta L ( 1 + \sqrt{M} \sigma_A / (1 - \sigma_A)) + \gamma T_C L \delta_0 \sigma_A^t  $, where $\delta_0 = \nrm{W_{0} - \one \Bar{w}'_0}$. Then from \eqref{eq:lem4_1} we can write :
\begin{equation} \label{eq:lem4_2}
    \nrm{\Bar{w}_{p} - w^*} \leq \Lambda^p \nrm{\Bar{w}_0 - w^*} + \sum_{l=0}^{p-1} \Lambda^{p-l} Y_l
\end{equation} 
With $\gamma < 1/(\mu T_C)$ we have $\Lambda < 1$. Using this and the fact that $\sigma_A < 1$ in \eqref{eq:lem4_2} we have : 
\begin{align} \label{eq:lem4_3}
    \nonumber \nrm{\Bar{w}_{p} - w^*} &\leq \Lambda^p \nrm{\Bar{w}_0 - w^*} + Y_0 \sum_{l=0}^{p-1} \Lambda^{p-l} \\ 
    \nonumber  &\leq  \Lambda^p \nrm{\Bar{w}_0 - w^*} + Y_0 / (1 - \Lambda)
\end{align}
\end{proof} 

%% ~~~~~~~~~~~~~~~  R:Main Result  ~~~~~~~~~~~~~~~ %%
\subsection{Main Result} \label{sec:Res_MnRes}
Here we present our main result on distributed federated learning using the proposed DFL algorithm in the following theorem. 
\begin{theorem} \label{thm:mainThm}
    Consider a distributed federated learning system  with $M$ servers, $N$ clients per server, where the communication among the servers is represented by graph $\mathcal{G}$. Suppose Assumptions \ref{asmpn:grphCnnctd}, \ref{asmpn:funcn_fij} and \ref{asmpn:gradBound} hold.
    The DFL algorithm in Algorithm \ref{alg:dfl}, with the step size $\gamma < \min\{1/LT_C, 1/\mu T_C\}$,  ensures that the prediction model across all the servers is within a tolerance value $\epsilon$ from the ideal model. Specifically, for all  $i \in \mathcal{S}$, 
    \begin{equation}
        \lim_{p \rightarrow \infty} \nrm{w^i_p - w^*} \leq \epsilon
    \end{equation} 
    where $\epsilon =  \sqrt{M} \gamma \theta T_C \sigma_A / (1-\sigma_A) + Y_0 / (1 - \Lambda)$, with $Y_0, \Lambda$ and $\sigma_A$ as in \eqref{eq:avg4mOptmlBound}. 
\end{theorem}

\begin{proof}
    For any server $i \in \srvr$ and epoch $p \in \N$ we have 
    \begin{equation} \label{eq:thm_1}
        \nrm{w^i_p - w^*} \leq \nrm{w^i_p - \Bar{w}_p} + \nrm{\Bar{w}_p - w^*}
    \end{equation}
    Using the results \eqref{eq:srvr4mAvgBound} from Lemma \ref{lem:srvrAvgBound} and \eqref{eq:avg4mOptmlBound} from Lemma \ref{lem:avg4mOptml} in \eqref{eq:thm_1} we get : 
    \begin{align}
        \nonumber \nrm{w^i_p - w^*} &\leq \sigma_A^p \nrm{W_{0} - \one \Bar{w}'_0} + \sqrt{M} T_C \theta \gamma \sigma_A / (1 - \sigma_A) \\ 
        \label{eq:thm_2} & \hspace{0.3cm} + \Lambda^p \nrm{\Bar{w}_0 - w^*} + Y_0 / (1 - \Lambda)
    \end{align} 
    Now using the fact that $\sigma_A < 1$ and $\Lambda < 1$, in the limiting case of \eqref{eq:thm_2} we have 
    \begin{equation}
        \lim_{p \rightarrow \infty} \nrm{w^i_p - w^*} \leq  \sqrt{M} T_C \theta \gamma \sigma_A / (1 - \sigma_A) + Y_0 / (1 - \Lambda)
    \end{equation}
\end{proof}

%% ~~~~~~~~~~~~~~~  Numerical Simulation  ~~~~~~~~~~~~~~~ %%
\section{Numerical Simulation} \label{sec:NumRes}
We present numerical simulation results considering a data fitting problem to illustrate the effectiveness of our novel DFL algorithm. We consider a system of 5 servers with 5 clients under each server. Further each client is allotted separate sets of 100 data-points each. All these 2500 data points are generated randomly such that $w^* = (5,2)$. The servers communicate among themselves over a connected graph. Within an epoch, we consider $T_C = 250$ iterations at the client and $T_S=25$ iterations at the server. The resultant straight line plot that we get from the model parameters generated by the DFL algorithm is shown in Fig.\ref{fig:eg_dataFit}(a). Consider the model parameter values at the servers over the $T_S$ iterations at every server in every epoch. Through Fig.\ref{fig:eg_dataFit}(b) we show that how each server starts off with quite different model parameter values based on their corresponding client model aggregation. Then after around 4000 server iterations, or 160 epochs, all the servers manage to achieve consensus over a common model parameter value, and this parameter value eventually comes close to the ideal values. This shows the effectiveness of using the consensus updates among the servers given in \eqref{eq:UL_server}.

\begin{figure}
    \centering
    %\begin{subfigure}[]
         %\centering
         \includegraphics[width=0.45\columnwidth]{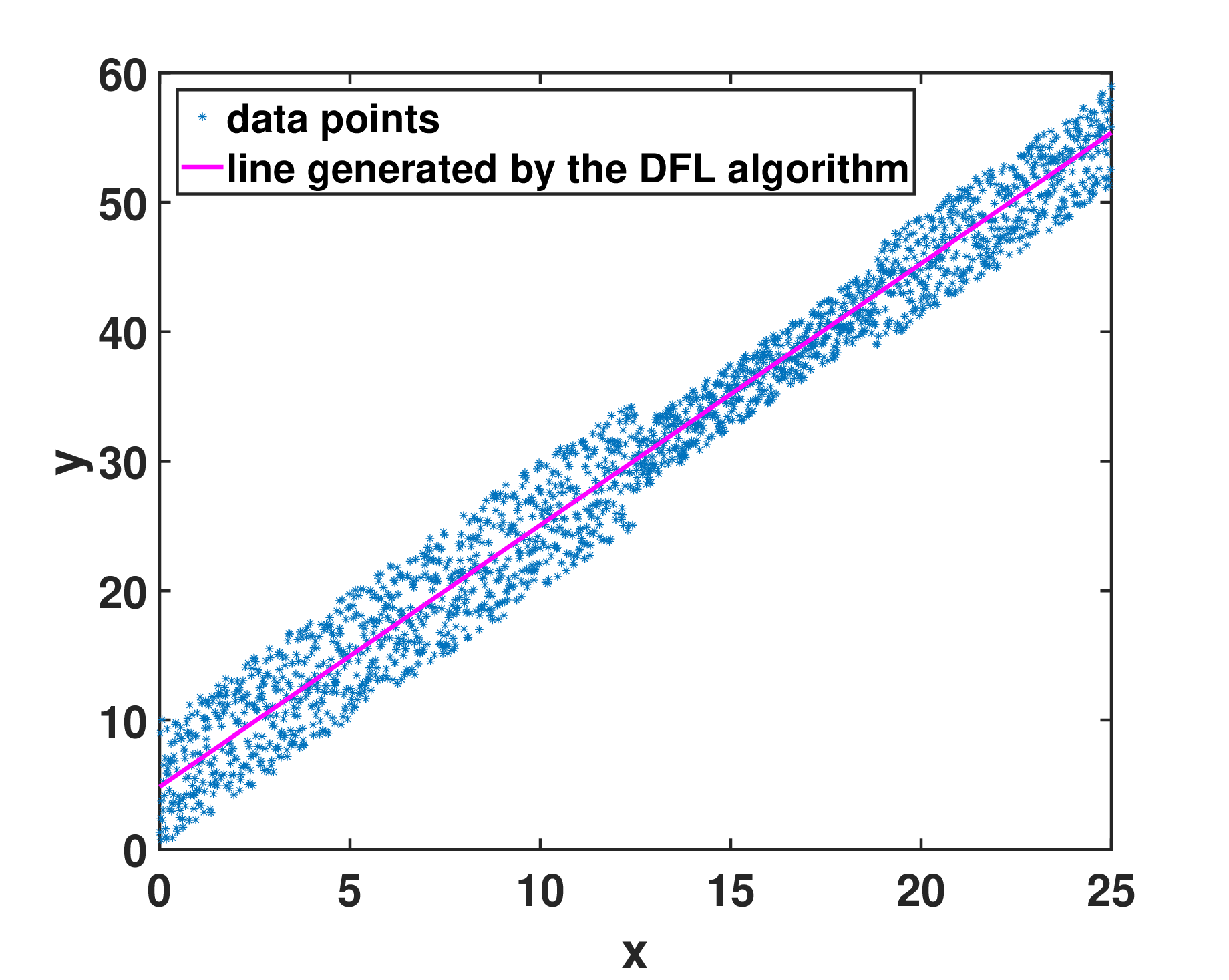}
         %\caption{}
         %\label{fig:1A_dataFit}
     %\end{subfigure}
    \hfill
   %\begin{subfigure}[]
         %\centering
         \includegraphics[width=0.45\columnwidth]{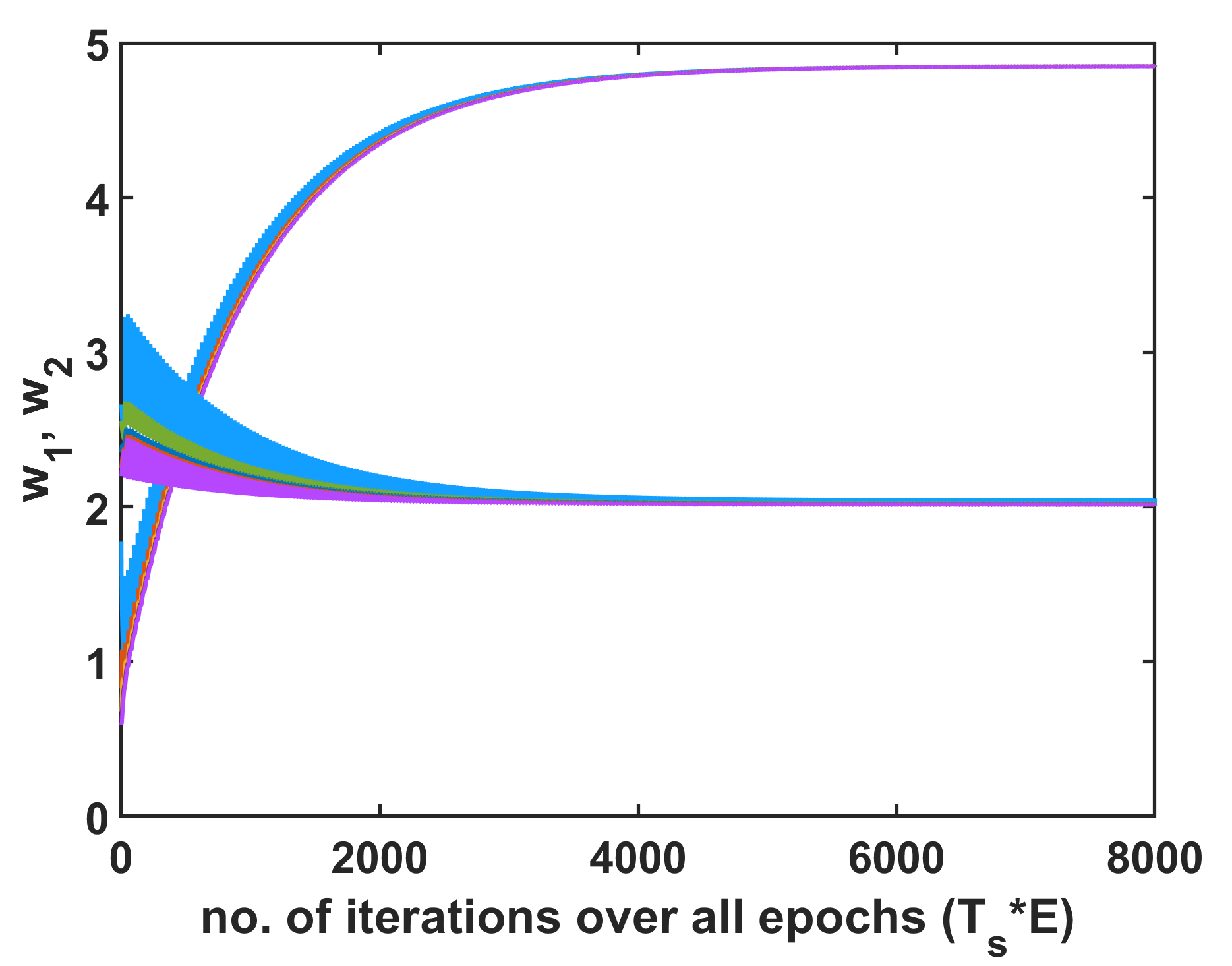} 
         %\caption{}
         %\label{fig:1B_srvrConsensus}
     %\end{subfigure}
    \caption{The DFL algorithm (a) generates a best-fitting straight line for the data across all clients, and (b) manages to get all the servers to achieve consensus.}
    \label{fig:eg_dataFit}
\end{figure}

%% ~~~~~~~~~~~~~~~  Conclusion  ~~~~~~~~~~~~~~~ %%
\section{Conclusion} \label{sec:Con}
In this paper we introduced a novel distributed federated learning system using multiple servers with a group of clients linked to each server. It addresses the challenges associated with having a single central server in the commonly used federated learning systems. In the proposed system with multiple servers, each server can communicate with its neighbouring servers, alongside communicating with its clients. A novel DFL algorithm is proposed which generates a common model parameter across servers trained on the data available across all clients. The DFL algorithm ensures that the sensitive user data remains with the clients and is not required to be shared with the server, remaining true to the main focus of federated learning algorithms of preserving user privacy. We established that under certain choice of parameters the proposed algorithm ensures that all the servers converge to a model value within a small tolerance from the ideal model. Finally we illustrated our result through a numerical simulation. 
As future work we would address communication challenges for this framework as addressed in the distributed optimization literature.
\balance

%\printbibliography

\bibliographystyle{IEEEtran} % Specifies the IEEE style
\bibliography{dfl_ref.bib} 

\vfill

\end{document}